\documentclass{svproc}
\usepackage{url}
\usepackage{amsmath}
\usepackage{filecontents}
\usepackage{tikz, amsfonts}
\usepackage{subfigure}
\usepackage{amssymb}
\usepackage{balance, dsfont, soul}
\usepackage{enumitem}  
\usepackage{algorithm, algpseudocode}
\usepackage[skins]{tcolorbox}
\usepackage{xcolor}

\DeclareMathOperator*{\argmax}{arg\,max}

\begin{document}
\mainmatter              % start of a contribution
\title{Budgeted Recommendation with Delayed Feedback}
\titlerunning{Budgeted Recommendation Delayed Feedback}  % abbreviated title (for running head)

\author{Kweiguu Liu\inst{1} \and Setareh Maghsudi\inst{2}}
\authorrunning{Liu et al.} % abbreviated author list (for running head)
%
%%% list of authors for the TOC (use if author list has to be modified)
\tocauthor{Ivar Ekeland, Roger Temam, Jeffrey Dean, David Grove,
Craig Chambers, Kim B. Bruce, and Elisa Bertino}

\institute{Kyushu University, Faculty of Information Science and Electrical Engineering, 819-0395 Fukuoka, Japan\\
\email{liu@agent.inf.kyushu-u.ac.jp, yokoo@inf.kyushu-u.ac.jp},
\and Ruhr-University Bochum, Faculty of Electrical Engineering and Information Technology, 44801 Bochum, Germany\\
\email{Setareh.Maghsudi@ruhr-uni-bochum.de}
}

% \institute{Princeton University, Princeton NJ 08544, USA,\\
% \email{I.Ekeland@princeton.edu},\\ WWW home page:
% \texttt{http://users/\homedir iekeland/web/welcome.html}
% \and Universit\'{e} de Paris-Sud,
% Laboratoire d'Analyse Num\'{e}rique, B\^{a}timent 425,\\
% F-91405 Orsay Cedex, France}

\maketitle              % typeset the title of the contribution

\begin{abstract}
In a conventional contextual multi-armed bandit problem, the feedback (or reward) is immediately observable after an action. Nevertheless, delayed feedback arises in numerous real-life situations and is particularly crucial in time-sensitive applications. The exploration-exploitation dilemma becomes particularly challenging under such conditions, as it couples with the interplay between delays and limited resources. Besides, a limited budget often aggravates the problem by restricting the exploration potential. A motivating example is the distribution of medical supplies at the early stage of COVID-19. The delayed feedback of testing results, thus insufficient information for learning, degraded the efficiency of resource allocation. Motivated by such applications, we study the effect of delayed feedback on constrained contextual bandits. We develop a decision-making policy, delay-oriented resource allocation with learning (DORAL), to optimize the resource expenditure in a contextual multi-armed bandit problem with arm-dependent delayed feedback. 
\keywords{Budget Constraints, Delayed Feedback, Online Learning, Resource Allocation}
\end{abstract}
%
%---------------------------------------------------->Introduction
\section{Introduction}
\label{sec:intro}
The contextual bandit problem is a well-known variant of the seminal multi-armed bandit problem: A decision-maker (agent, hereafter) observes some random context, i.e., features, at each round of decision-making. It then pulls one of the available arms and immediately receives a reward generated by the random reward process of the selected arm. Given the context, the agent maximizes each round's reward while effectively exploring the potential alternatives. The state-of-the-art applications of the contextual bandit problem include online advertising and personalized recommendation; nevertheless, the current research neglects those applications with additional constraints on resources \cite{badanidiyuru2014resourceful} that cause the following challenges: (i) The exploitation becomes limited by the resource available for exploration; (ii) The agent can no longer seek to maximize the instantaneous rewards as the arm with the highest reward can be an expensive one. Consequently, the work opted for maximizing the accumulated rewards.  One main assumption in budgeted learning is immediately observable feedback, but feedback is usually late in real-world applications.  Delayed feedback exacerbates the difficulty of exploration because the information about suboptimal arms procrastinates. Consequently, delayed feedback makes resource allocation inefficient for exploration and exploitation. A motivating example is the distribution of medical supplies at the early stage of COVID-19. During the outbreak, medical supplies, e.g., protective kits and ventilators, do not support urgent needs. That renders an optimal allocation of scarce resources imperative; nevertheless, delayed feedback about testing results hinders health administrators from accomplishing that goal, as inaccurate estimates aggravate the difficulty of making decisions for vital resources. To address this real-world challenge, the Greek government collaborated with the schools from the United States; \cite{bastani2021efficient} designed and deployed a national scale learning system named Eva to save lives. Eva's goal is to efficiently allocate scarce testing resources to identify as many infected passengers as possible while striving for more accurate estimates of COVID-19 prevalence from passengers. The challenge Eva faced was the delayed feedback of COVID test results because delayed feedback brings the following adverse effects in budgeted learning:
\begin{enumerate} 
    \item \textbf{Over-exploration:} Delayed feedback yields an inaccurate estimation of unknown parameters. Thus, the agent over-explores to retrieve the information that it could have simply obtained in a non-delayed scenario.  
    \item \textbf{Inefficient allocation:} Over-exploration results in ineffective budget expenditure because over-exploration depletes the resources that could have been used to explore arms with higher rewards and a longer response time.  
    \item \textbf{Ineffective allocation:} The delayed feedback in Eva was guaranteed to return in a shorter period. However, feedback in other real-world applications is usually not guaranteed to return, and such unresponsive feedback makes resource allocation challenging.   
\end{enumerate}
To tackle the challenges above, one shall incorporate delays in budget planning, i.e., one needs to consider expected delayed rewards of arms because resource allocation is also affected by delayed feedback.  Hence, we disentangle the adverse effects of delayed feedback by gradually filtering out less-responsive arms, i.e., arms with excessively long or unbounded delays.  Next, we formulate a delay-oriented linear program to handle online resource allocation.  Specifically, our proposed algorithm \emph{delay-oriented resource allocation with learning} (DORAL) consists of two stages: (i) using a fraction of the budget, the first stage identifies a set of top responsive arms that are likely to return feedback within a time window; (ii) with the remaining budget, the second stage uses the obtained set of top responsive arms to form delay-oriented linear programming to optimize resource allocation. 

\paragraph{CONTRIBUTIONS.} Our proposed two-stage algorithm ensures an efficient online resource allocation in a general setting, where arm-dependent delays can be excessively long or unbounded. The previous works in constraint learning circumvent the issues caused by delayed rewards via posterior sampling or patiently waiting for feedback. Such remedies are applicable given prior information about delays or when the waiting time is relatively short. Our alternative solution handles delayed feedback directly through the joint allocation of resources and learning time. Also, we propose a delayed version of the robust top arms identification method.
%-------------------------------------------------------->Related Work
\section{Related Work}
\label{sec:related_work}
Recent research on learning with delayed feedback includes diverse applications such as \cite{chapelle2014simple,vernade2017stochastic} on personalized recommendation, \cite{chen2019task,ghoorchian2020multi} on edge computing,\cite{amuru2014optimal} on the military, and \cite{cesa2018nonstochastic} on communication networks. The cutting-edge research categorizes delayed feedback into two classes, namely, bounded- and unbounded delays. For example, a medical result that arrives within $48$ hours is a bounded delay, whereas a customer's feedback that usually never returns is unbounded. The state-of-the-art methods combine existing learning algorithms with the following concepts to handle delayed feedback: \emph{waiting} or \emph{cut-off}. The cut-off is a predetermined waiting window discarding any feedback outside its boundaries. Waiting for feedback to update the estimators is a popular method; nonetheless, it is appropriate when delays are bounded \cite{joulani2013online,zhou2019learning}.  In cases of significant delays, cut-off is the best fit, because waiting indefinitely without updating estimators can worsen bias and increase storage overhead. Reference \cite{vernade2020linear} applies the concept to delayed linear bandits. Similar applications appear in \cite{thune2019non}. 

Resource allocation with delayed feedback has gained attention due to various challenges in real-world applications. One solution to allocate resources with delays is via posterior sampling to estimate delayed feedback. The Greek government and \cite{bastani2021efficient} designed the system Eva for the urgent allocation of limited medical resources during the COVID outbreak, and Eva circumvents delayed feedback, i.e., testing results, by applying the empirical Bayes procedure to estimate the prevalence of passengers from different countries.  Similarly, Northern American ride-sharing company Lyft allocates resources to ad campaigns across periods to attract more drivers, but prospective drivers cannot hit the road until finishing the mandatory requirements. Thus, \cite{han2020contextual} applies Thompson Sampling to predict potential drivers, i.e., delayed rewards for resource allocation. Compared to the previous work, our work considers possible non-returned feedback that makes allocated resources ineffective, so we need to identify which arms are less responsive to minimize ineffective resource allocation.  
%-------------------------------------------------->Problem Formulation
\section{Problem Formulation}
\label{sec:problem}
We consider an environment with finite classes of contexts $\mathcal{X} = \{1,\dots, J\}$. Let $\pi_1,\dots,\pi_J$ denote a known distribution over contexts. Each type of context is characterized by an unknown parameter vector $\theta_j \in \mathbb{R}^k$, where $k$ denotes the number of features. A finite set of arms (actions) is $\mathcal{A}$, and $|\mathcal{A}| = A$. At each time $t$, a subset of $\mathcal{A}$, denoted by $\mathcal{A}_{t}$, is available. Each arm $a_t \in \mathcal{A}_t$ has a feature vector $\mathbf{f}_{a_t} \in \mathbb{R}^d$ and a corresponding fixed cost $c_{a_t}$; $\mathbf{f}^{(j)}_{a_t}$ denotes the feature of $a_t$ selected for context $j$. Each arm has a delay distribution $\mathcal{D}_{a_t}$ with unknown mean $d_{a_t}$ supported on positive numbers, and delay $D_{a_t} \sim \mathcal{D}_{a_t}$ is sampled.  Because arms have different delays due to exogenous factors, $D_{a_t}$ does not depend on the contexts. Initially, the agent has some budget $B$. At each round $t \in T$, the remaining budget is $b_t$. The agent interacts with the environment as follows until the budget is exhausted. 
\begin{tcolorbox}[standard jigsaw, opacityback=0]
While $b_t > 0,$
\begin{itemize}
    \item At time $t$, the agent observes a context from $j \in \mathcal{X}$, where the contexts arrive independent of each other and the set of available arms $\mathcal{A}_t$; 
    \item The agent selects an arm $a_{t} \in \mathcal{A}_t$ to maximize the weighted reward $r^{(j)}_{t, u, a_t} = \langle \theta_j, \mathbf{f}_{a_t}\rangle$, where $u$ is the latest time in which the reward shall become observable. The agent pays the associated cost $c_{a_t}$ or not if she pulls no arm.  The agent is allowed to skip any round, i.e., pull no arm, whenever no arm can be recommended given the remaining budget.
\end{itemize}
\end{tcolorbox}
For $u > 0$, $\mathbb{P}(D_{a_t} \leq u) = \tau_{a_t}(u)$. Following \cite{gael2020stochastic}, we assume $\tau_{a_t}(u)$ satisfying the following inequality. Let $\alpha > 0$, 
\begin{equation}
    |1-\tau_{a_t}(u)| \leq u^{-\alpha}
\end{equation}
By the assumption above, a smaller $\alpha$ implies a lower chance of receiving feedback by $u$. When we consider the case with heavy-tailed delays, we assume $\mathbb{E}|D_{a_t} - d_{a_t}|^{1+\varepsilon} \leq v_{a_t} $, where $\varepsilon \in (0,1]$ and $v_{a_t}$ denotes variance. In other words, delays are not assumed to be sub-Gaussian. Also, we assume each arm's mean delay is not larger than a certain portion of the budget, i.e., $d_{a_t} \leq B/4$, because any mean delay larger than the budget is rarely observed and can be discarded.  The agent does not know feedback before $D_{a_t}$ exceeds $u-t$. Thus, we define the delayed reward formally as
\begin{equation}
\label{eq:DelayedReward}
\hat{r}_{t,u,a_t} = r_{t,u,a_t}\mathds{1}\{D_{a_t} \leq u-t\}
\end{equation}
where $\mathds{1}\{D_{a_t} \leq u-t\}$ is an indicator function that returns one if the reward of the decision made at $t$ is observed by $u$ and zero otherwise. The agent's objective is to select the arms sequentially to maximize the total accumulated reward given the budget constraint and delayed feedback. Formally,   
\begin{equation}
    \label{eq:goal}
    \begin{split}
        \text{maximize} \quad & U(T,B) = \mathbb{E}\left[\sum_{t=0}^T\hat{r}_{t,u,a_t}\right]\\
        \text{subject to} \quad & \sum_{t=1}^T c_{a_t} \leq B
    \end{split}
\end{equation}
where the expectation is taken over the distribution of contexts and rewards.  Let $U^*(T,B) = \mathbb{E}[\sum_{t=0}^T\hat{r}_{t,u}^*]$ denote the total optimal payoff when $\hat{r}_{t,u}^* = \max\limits_{a_t\in\mathcal{A}_t}\hat{r}_{t,u,a_t}$.  We measure the performance by the regret, i.e., the difference between the expected gain with hindsight knowledge and the actual gain, and it is defined as 
\begin{equation}
\label{eq:regret}
R = U^*(T,B)-U(T,B)
\end{equation}
The agent minimizes the regret by optimal arm selection.
%-------------------------------->Algorithm Design
\section{Algorithm Design}
\label{sec:Alg}
The proposed algorithm consists of two stages. The first stage is to identify a set of top responsive arms. We describe this stage in \textbf{Sec.~\ref{sec:search_responsive_arms}}. The second stage is online allocation with learning; we explain this stage in \textbf{Sec.~\ref{sec:resources with delays}}.
%------------------------------------>Search for Top Responsive Arms
\subsection{Search for Top Responsive Arms}
\label{sec:search_responsive_arms} 
To mitigate the adverse effects of arm-dependent delays on resource allocation, we need to know the arms' response time in order to determine the cut-off $m$ for contextual learning in the next stage.  Hence, we first need to identify the top responsive $A' \leq A$ arms.  This task boils down to an identification problem in multi-armed bandits, where the rewards are the average over the number of rounds in which the agent observes delays.  Also, ranking arms according to their responsiveness enables us to neglect the ones with rare or no feedback, thus saving scarce resources for the rest. 

We build our identification algorithm upon a strategy family known as \emph{Successive Acceptance and Rejection} (SAR) from \cite{bubeck2013multiple,grover2018best}. Different variants of SAR can optimize either the budget for a given confidence level or the quality of exploration for a given threshold. However, in our setting, delays make it challenging to decide on a necessary budget in the first place. To address this issue, inspired by \cite{heidrich2009hoeffding}, we propose a variant of SAR, namely, \textit{Patient-Racing SAR} (PR-SAR), for a given threshold. \textbf{Algorithm~\ref{algo:racing_SAR}} describes PR-SAR.  

Let $S_t$ denote a set of accepting arms, $L_t$ a set of remaining arms at time $t$, and $E_t$ a set of rejected arms. Till identifying top responsive arms, $|S_t| = A'$, PR-SAR continuously pulls arms from $L_t$, and the method determines acceptance by comparing confidence bounds. However, due to \emph{delayed response} and possible \emph{non-sub-Gaussian delays}, we need a robust estimation method, namely, \emph{median-of-means} estimator by \cite{bubeck2013bandits}, to measure the responsiveness of arms. The core idea is to prepare several disjoint baskets, calculate the standard empirical mean of received feedback in each basket and take a median value of these empirical means. 
% Let $X \wedge Y  \triangleq$ 
% $\min(X, Y), $
$X \vee Y \triangleq $
% $\max(X, Y)$. 
Specifically, for some arm $a$, $T_a(u) = \sum_t^u\mathds{1}\{a_t = a\}$
denote the number of times the agent observes arm $a$ by time $u$, and $h = \lfloor8\log(\frac{e^{1/8}}{\delta} \wedge \frac{T_a(u)}{2}))\rfloor$ and $N(u) = \lfloor \frac{T_a(u)}{h} \rfloor$. Let $D_{a,t} = D_{a_t}\mathds{1}\{a_t = a\}$. Each basket's estimated expected waiting period then yields
\begin{equation*}
    \hat{d}_{a,1} = \frac{1}{N(u)}\sum_{t=1}^{N(u)}D_{a, t}, \hat{d}_{a,2} = \frac{1}{N(u)}\sum_{t=N(u)+1}^{2N(u)}D_{a,t}, \dots, \hat{d}_{a,h} = \frac{1}{N(u)}\sum_{t=(h-1)N(u)+1}^{hN}D_{a, t}.
\end{equation*}
Let $d_a^M$ denote the median-of-means estimator of empirical waiting periods. We first need the following lemma to bound $d_a^M$; the lemma states how the empirical mean behaves when delayed feedback exists.  Due to the limited space, the proof is omitted.   
\begin{lemma}
\label{thm:empirical_delayed_mean}
    For some arm $a$ and $\alpha, \delta > 0$, with probability at least $1-\delta - B^{-\alpha}$, $\hat{d}_a \leq d_a + \sqrt{\frac{2B\log\frac{2}{\delta}}{T_a(u)}} + 2d_aT(u)^{-(\alpha \wedge 1/2)}$ 
\end{lemma}
The following theorem states the robust upper confidence bound (UCB) of patient median-of-means estimation. 
\begin{theorem}
\label{thm:delayed_robust_UCB}
Let $\alpha > 0, \delta > 0$. For some $a$ and any $t > A$ with probability $1 - \delta - B^{-\alpha}$,
\begin{equation}
    |d_a^M - d_a| \leq \sqrt{\frac{2\log\left(\frac{16}{1-B^{-\alpha}}\right)}{T_a(u)}} + \frac{B}{2}T_a(u)^{-(\alpha \wedge 1/2)}
\end{equation}
\end{theorem}
\begin{proof}
Let $Z_l = \mathds{1}\{\hat{d}_l > d_a + \varepsilon,  \forall l \in \{1,\dots,h\}\}$. According to \textbf{Lemma~\ref{thm:empirical_delayed_mean}}, $Z_l$ follows a Bernoulli distribution with $p \leq \frac{v_a}{N(u)^{\varepsilon}\zeta^{1+\varepsilon}} + 2\exp{\left(\frac{-2\zeta^2 N(u)}{B^2}\right)} + B^{-\alpha}$.  If $\zeta = \sqrt{\frac{2\log\left(\frac{16}{1-B^{-\alpha}}\right)}{N(u)}}$, we have $p \leq 1/4 + B^{-\alpha}$. By Hoeffding's inequality, we have 
\begin{equation*}
    \mathbb{P}(d^M_a > d_a + \varepsilon) = \mathbb{P}\left(\sum_{l=1}^h Z_l \geq \frac{h}{2}\right ) \leq \exp{(-2h(1/2-p)^2)} \leq \exp{(-h/8)} \leq \delta
\end{equation*}
\end{proof}

According to \textbf{Theorem}~\ref{thm:delayed_robust_UCB}, PR-SAR accepts any arm whose lower confidence bound (LCB) is larger than at least $K'$ of UCB's.  Specifically, PR-SAR compares the following delayed robust UCB's and LCB's. 
% ------------------------
\begin{equation}
    \label{eq:robust_UCB}
    UCB_{d_a}(t) = d^M_a(t) + \sqrt{\frac{2\log\left(\frac{16}{1-B^{-\alpha}}\right)}{T_a(u)}} + 2d_aT_a(u)^{-(\alpha \wedge 1/2)} 
\end{equation}
\begin{equation}
    \label{eq:robust_LCB}
    LCB_{d_a}(t) = d^M_a(t) - \sqrt{\frac{2\log\left(\frac{16}{1-B^{-\alpha}}\right)}{T_a(u)}} - 2d_aT_a(u)^{-(\alpha \wedge 1/2)}
\end{equation}
After finding top responsive arms, PR-SAR determines the remaining budget for resource allocation, i.e., $B_{ac} = B - B_{id}$ where $B_{id}$ is the amount of budget spent on identification. PR-SAR can simply decides the cut-off  $m = \max_{a\in\mathcal{A}}UCB_{d_a}$.  PR-SAR by construction is Hoeffding race method, so it can find top responsive arms.
\algrenewcommand\algorithmicrequire{\textbf{Input:}}
\algrenewcommand\algorithmicensure{\textbf{Output:}}
\begin{algorithm}[!ht]
\caption{Patient Racing SAR}\label{algo:racing_SAR}
\begin{algorithmic}
    \Require $\mathcal{A}, B, S_1 = \{\emptyset\}, L_1 = \{1, \dots, K\}$
    \Ensure $A^{\prime}$ accepted arms and $m = max\{UCB_{d_a}, \forall a \in \mathcal{A}\}$
    \While{$|S_t| < A'$}{
	\For{$a \in L_t$}
        \State pull $a$ and compute $UCB_{d_a}$ using \textbf{Eq.}~\ref{eq:robust_UCB}, $LCB_{d_a}$ using \textbf{Eq.}~\ref{eq:robust_LCB}
    \EndFor
    \For{$a \in L_t$}
        \Comment{Update top accepted arms}
        \If{$|\{a' \mid LCB_{d_a} > UCB_{d_{a'}}\}| > A' - |S_t|$}
            \State $S_{t} \leftarrow S_{t} \bigcup \{ a\}$
            \State $L_{t} \leftarrow L_{t} \backslash \{ a\}$ 
        \EndIf
        % \If{$|\{a \mid LCB_{d_{a'}} < UCB_{d_{a}}\}| > |S_{t}|$}
        %    \State $L_{t} \leftarrow L_{t} \backslash \{ a\}$
        % \EndIf
    \EndFor
    }
    \EndWhile
\end{algorithmic}
\end{algorithm}
%---------------------->Resource Allocation with Delays
\subsection{Resource Allocation with Delays}
\label{sec:resources with delays}
We first introduce the decision rule when delayed feedback exists, and then explain online resource allocation with delayed feedback.
%------------------------->Learning Estimators
\subsubsection{Learning Estimators with Delayed Feedback}
Waiting indefinitely for feedback results in excessive computation overhead and swift exhaustion of scarce resources.  One solution to mitigate the problem is to select a cut-off parameter $m$. Thus the delayed reward $\hat{r}_{t,u,a}$ can be restated as
\begin{equation}
\Tilde{r}_{t,u,a} = r_{t,u,a}\mathds{1}\{D_{t,a} \leq \min(m, u-t)\}
\end{equation} 
Let $\lambda > 0$ be a regularization parameter. Following \cite{vernade2020linear}, we estimate $\theta$ using the least square method as 
\begin{equation}
\label{eq:estimated_theta}
\hat{\theta}_t = \left(\sum_{t=1}^{u-1} \mathbf{f}_{a_t} \mathbf{f}_{a_t}^{\top}+\lambda I\right)^{-1}\left(\sum_{t=1}^{u-1} \Tilde{r}_{t,u,a} \mathbf{f}_{a_t}\right) = V_t(\lambda)^{-1}G_t.
\end{equation}
Let $f_{t,\delta}=\sqrt{\lambda} + \sqrt{2\log\left(\frac{1}{\lambda}) + k\log(\frac{k\lambda+t}{k\lambda}\right)}$.  Theorem 1 in \cite{vernade2020linear} validates that, after a period of learning, the distance between $\hat{\theta}$ and $\theta$ remains bounded with high probability.  Hence, we propose the following decision rule to select arms. 
%-------------------------->Decision Rule
%
\begin{tcolorbox}[standard jigsaw, opacityback=0]
\begin{itemize}
\item At each round and for each arm $a$, define the index $\gamma_{t}(a)$ as 
\begin{equation}
\label{eq:delayed_LinUCB}
\langle\hat{\theta}_{t}, \mathbf{f}_a\rangle + \left(2f_{t,\delta} + \sum_{t=u-m}^{u-1}\|\mathbf{f}_{a_t}\|_{V_t(\lambda)^{-1}}\right)\|\mathbf{f}_{a}\|_{V_t(\lambda)^{-1}}
\end{equation} 
The index is the linear upper confidence bound (LinUCB) within the cut-off.    
\item The agent picks the arm with the highest index, i.e., $a_{t}^*=\argmax{\gamma_t(a)}$, $a \in \mathcal{A}_{t}$.
\end{itemize} 
\end{tcolorbox}

%----------------------------------->Subsection:Delayed RA
\subsubsection{Resource Allocation with Delayed Learning}
\label{sec:full_method}
Optimization of resource allocation in (\ref{eq:goal}) with the hard budget constraint is especially challenging with unknown delays. Hence, we approximate the optimal solution with a relaxed budget constraint, i.e., the average budget constraint $\rho = \frac{B}{T}$ at each round. Inspired by the approximation method in \cite{wu2015algo}, we develop an approach for near-optimal delay-oriented allocation. In the following, we drop the parameter $m$ from $\tau_a(m)$ unless it is necessary to avoid ambiguity.  For simplicity, we assume that the distribution of contexts is known and static. We have $\tau_a \rightarrow 1$ if $m \rightarrow \infty$, but such $m$ increases the regret significantly. To simplify the analysis, we assume $\tau_a$'s are given. Let $\eta^*_j = \max_{a \in \mathcal{A}}\tau_a\Tilde{r}_{j,a}$ the best expected delayed reward the agent can obtain under context $j$ and $\Tilde{a}^*_j = \argmax_{a \in \mathcal{A}}\Tilde{r}_{j, a}$ the corresponding arm.  Let $ \mathbf{p}= (p_1, \dots, p_j)$ denote a probability vector, and the agent's goal is to solve the following linear programming at each round: 
\begin{equation}
\label{eq:lp_m}
\begin{split}
% \text{$LP_m$~~~maximize}_{\mathbf{p}} \quad & \sum_j^J p_j \pi_j \textcolor{blue}{(d^*_j + \Tilde{r}^*_j)} \\
\text{$LP_m$~~~maximize}_{\mathbf{p}} \quad & \sum_j^J p_j \pi_j \eta^*_j \\
\text{subject to} \quad & \sum_j^J p_j \pi_j \leq \rho \\
\end{split}
\end{equation}
$LP_m$ maximizes the expected delayed reward with arms with higher probabilities to return, while its constraint considers expected delayed costs to avoid spending resources on arms with no feedback possibilities.  The solution of $LP_{m}$ can be expressed with some threshold $j(\rho)$, which is a function of the average constraint ratio $\rho$, and the reinterpretation of $LP_m$ can help simplify the regret analysis. Thus,    
\begin{equation}
\label{eq:threshold_LP_m}
j(\rho) = \max\{j: \sum_{j'}^{J} \pi_{j'} \leq \rho\}
\end{equation}
\begin{equation}
\label{eq:ranking_LP_m}
p_j(\rho)= 
\begin{cases}
1, & \text{if } 1 \leq i \leq j(\rho) \\
\frac{\rho - \sum_{j^{\prime} = 1}^{j(\rho)} \pi_{j^{\prime}}} 
{\pi_{j(\rho)+1}}, & \text{if } i=j(\rho)+1\\
0, & \text{if } i > j(\rho)  + 1 
\end{cases}
\end{equation}
and the optimal value of $LP_m$ can be expressed with (\ref{eq:threshold_LP_m}) and (\ref{eq:ranking_LP_m}).
\begin{equation}
\label{eq:delayed_optimum}
    v(\rho) = \left(\sum_{j'}^{j(\rho)} \pi_{j'} \hat{\eta}_{j'}^*\right) + p_{j(\rho) + 1} \pi_{j(\rho) + 1} \eta^*_{j(\rho)+1}
\end{equation} 

With \textbf{Algorithm~\ref{algo:racing_SAR}} and the resource allocation rule (\ref{eq:lp_m}), we present our proposed decision-making strategy, i.e., delay-oriented resource allocation with learning (DORAL), in \textbf{Algorithm~\ref{algo:full}}. The algorithm consists of two stages.  The first stage spends a portion of the budget $B_{id}$ and time $T_{id}$ to identify top responsive arms, and the second stage explores and exploits with the remaining budget $B_{ac}$ and time $T_{ac}$ using the cut-off obtained in the first stage.  

\begin{algorithm}[!h]
\caption{Delay-Oriented Resource Allocation with Learning}\label{algo:full}
\begin{algorithmic}
\Require $\mathcal{A}, S_1 = \{\emptyset\}$, $B$, $T$, $\lambda \in (0,1)$
\While{$|S_t| \leq A'$}
    \State Algorithm~\ref{algo:racing_SAR}
\EndWhile
\While{$B_{ac} > 0$}
    \State Observe context and pick $\Tilde{a}^*_j(t)$ with $p_j(\frac{b_t}{t})$ that satisfies $LP_m$ with $\hat{\eta}^*_j$
    \State  Update $b_t, \hat{\theta}_j, V^{(j)}_{\lambda}(t)$, and $G^{(j)}_t$ if delayed feedback returns.
\EndWhile
\end{algorithmic}
\end{algorithm}

%--------------------------------------->Experiments
\section{Experiments}
Due to the difficulty of finding a dataset suitable for our scenarios, we evaluate the performance of DORAL with a synthetic dataset with the following settings: The budget is $85,000$ to ensure at least $50,000$ rounds. There are $10$ context classes, and the distribution of the contexts is $\boldsymbol{\pi} = [0.09, 0.15, 0.11, 0.05, 0.1, 0.05, \\0.08, 0.14, 0.13, 0.1]$. There are $10$ arms, and each has a unit cost. Each context and arm has five features, where we represent each feature by some value between $(0,1)$. We use geometric and Pareto distributions to generate delays for each assigned arm. For the scenario of diverse delays, we have geometric delays for the arms, and their expected delays are $[100, 120, 140, 160, 200, 220, 240, 260, 280, 300]$.  For Pareto delays, we set each arm's minimum delay as $[200, 220, 240, 260, 280, 320, \\340, 360, 380, 400]$, and the arms share the same shape parameter $\alpha = 2$.  
For the scenario of similar and short delays, we use [100,110,120,130,140,150,160,170,180,190] for geometric means and Pareto minimum values.  We compare our proposed algorithm with the following benchmarks.  \textbf{Delayed-LinUCB (D-LinUCB)} by \cite{vernade2020linear}: The method selects the arms with the highest reward in each context class. The method selects the arms with the highest reward in each class of context.  Because the type of delay distribution is unknown, we choose $m=500$ for the cut-off $m$. This method can be interpreted as a greedy method in our scenario. Thus, we have the lowest $\tau_a(m) = 0.81$ in the case of geometric delays, and the lowest $\tau_a(m) = 0.36$ in the case of Pareto delays. \textbf{Random Delayed-LinUCB (Random)}: The method is similar to Delayed-LinUCB; Nevertheless, the selection follows with the probability of $\rho$, i.e., the remaining budget concerning the given budget. It also uses $m=500$ for the cut-off.  \textbf{Delayed adaptive linear programming (D-ALP)} of \cite{wu2015algo}: The method is similar to our proposed method, but we let $\forall a \in \mathcal{A}, \tau_a(m)=1$. Each figure is the average of 50 runs. 
As DORAL spends some time identifying top responsive arms, it starts late to accumulate rewards. In \textbf{Fig.~\ref{fig:result-1}}, because delays are short, DORAL and D-ALP are overlapping. 
In \textbf{Fig.~\ref{fig:result-2}} (a), D-ALP identifies a higher $m > 500$, so it's expected to have more regret according to the theorems. Nevertheless, DORAL outperforms D-ALP when facing heavy-tailed delays in \textbf{Fig.~\ref{fig:result-2}} (b) even though both of them identify similar $m$, while they come close in  \textbf{Fig.~\ref{fig:result-2}} (c).  These cases indicate that utilizing expected delayed rewards in diverse delays can simultaneously optimize rewards and learning altogether.
%--------------------------Fig
\begin{figure}[t]
    \centering
    \subfigure[Geometric Delays]{\includegraphics[scale=0.35]{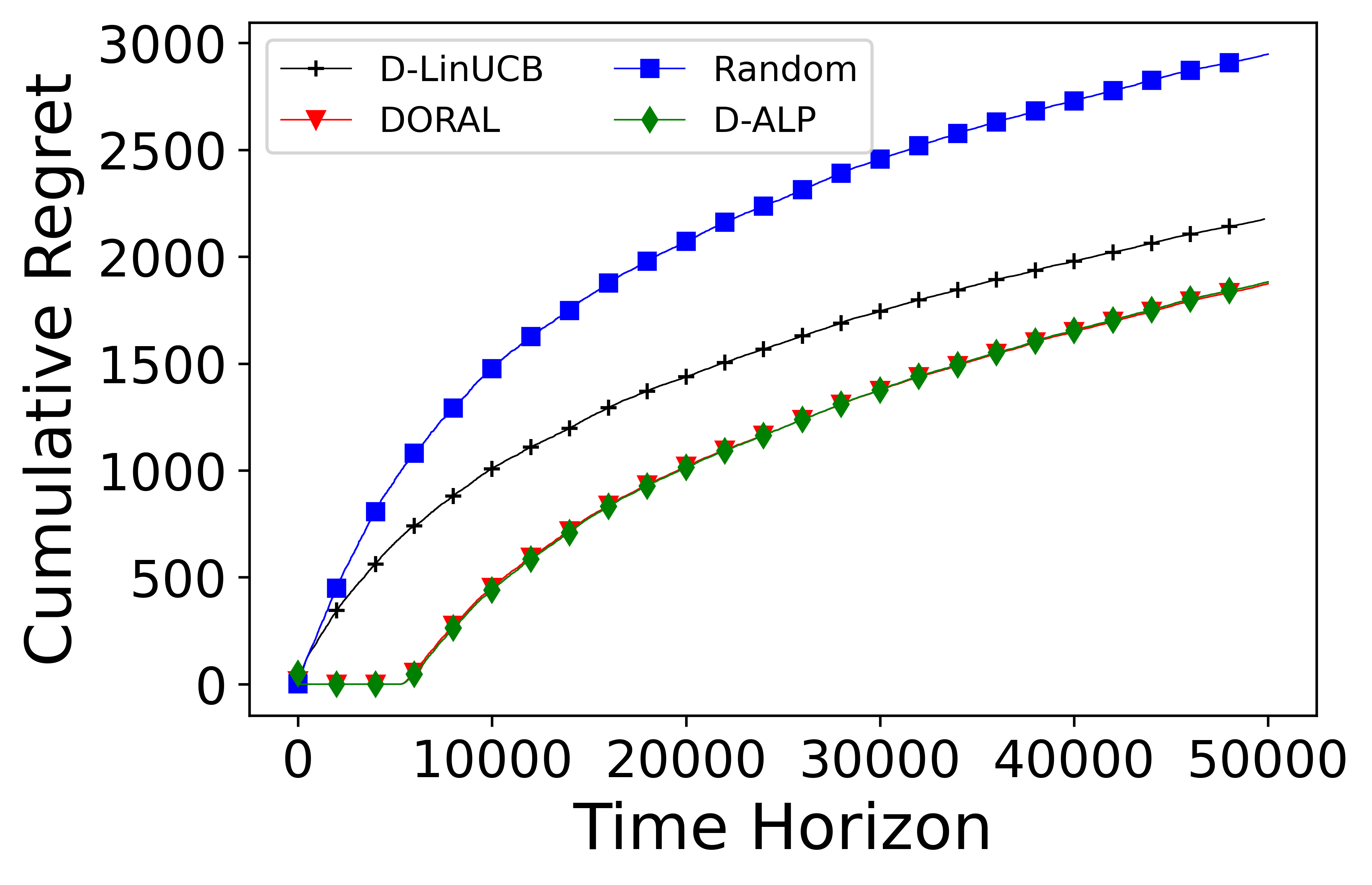}}
    \subfigure[Pareto Delays]{\includegraphics[scale=0.35]{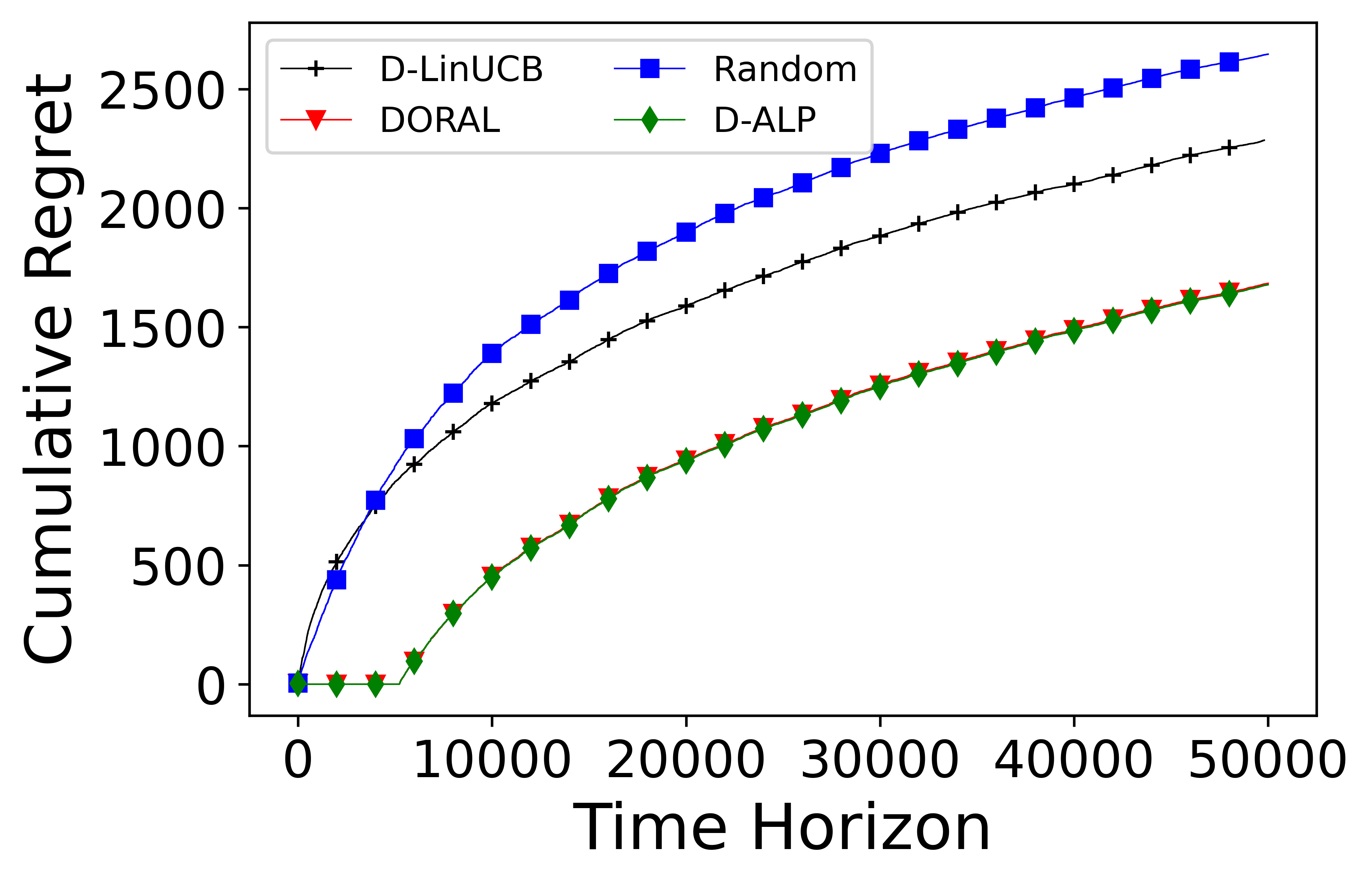}}
\caption{Similar delays}
\label{fig:result-1}
\end{figure}
\begin{figure}[t]
    \centering
    \subfigure[Geometric Delays]{\includegraphics[scale=0.35]{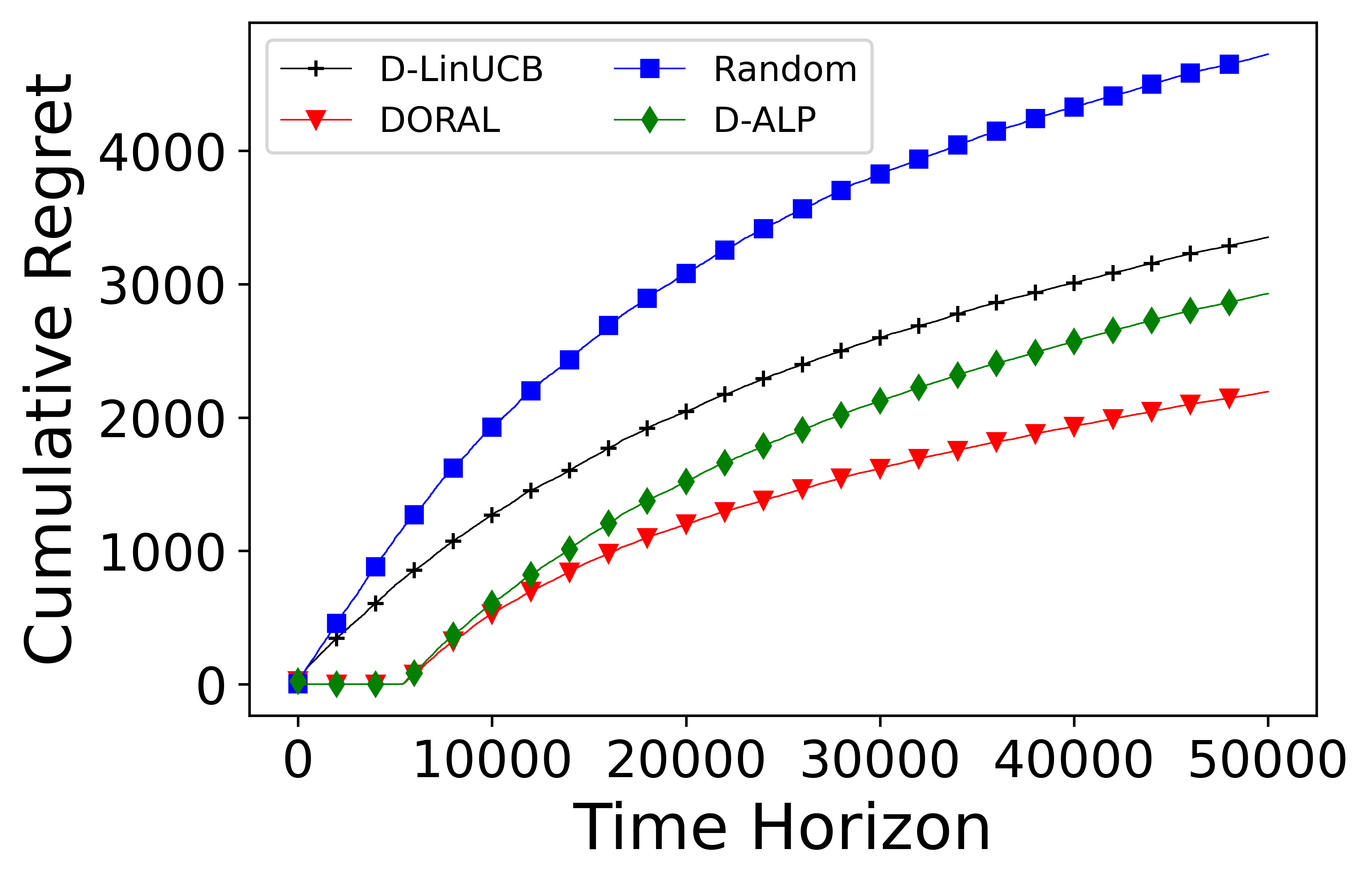}}
    \subfigure[Pareto Delays]{\includegraphics[scale=0.35]{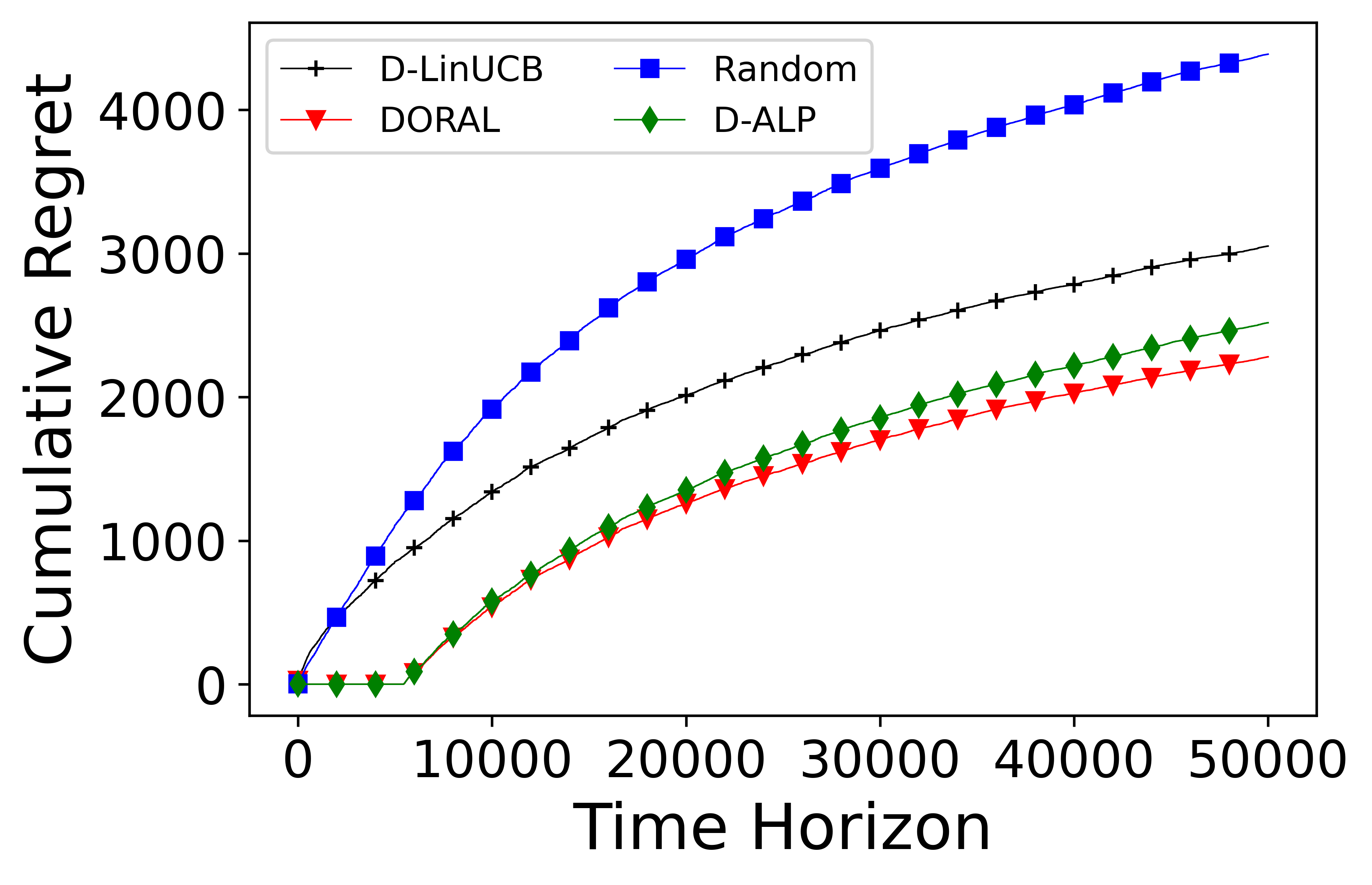}}
\caption{Diverse delays}
\label{fig:result-2}
\end{figure}
%------------------------------------->Conclusion
\section{Conclusion}
To tackle the challenges in resource allocation with delayed feedback, we developed a two-stage policy that can efficiently allocate resources while learning with delayed feedback. Also, we proposed a robust method to identify top responsive arms when information delays can be heavy-tailed. Future research involves simultaneously determining cut-off on the fly while ensuring efficient resource allocation. Context-dependent delays are more realistic and challenging when compared to arm-dependent delays in our setting. Also, in our setting, the feedback of different arms is equally important, although often, there exist different levels of urgency to consider in a resource allocation problem, e.g., in sharing limited medical supplies. Hence, another research direction is studying such a hierarchical structure in real-world applications.

\section*{Acknowledgement}
The work of S.M. was supported by Grant  01IS20051 and Grant 16KISK035 from the German Federal Ministry of Education and Research.

\bibliographystyle{splncs_srt}
\bibliography{world2024}

\end{document}